\documentclass{article}

\usepackage{microtype}
\usepackage{graphicx}
\usepackage{booktabs} 
\usepackage{caption}
\usepackage{subcaption}
\usepackage{hyperref}
\usepackage{algpseudocode}

\usepackage[accepted]{icml2023}

\usepackage{amsmath}
\usepackage{amssymb}
\usepackage{mathtools}
\usepackage{amsthm}

\usepackage[capitalize,noabbrev]{cleveref}

\theoremstyle{plain}
\newtheorem{theorem}{Theorem}[section]
\newtheorem{proposition}[theorem]{Proposition}

\theoremstyle{definition}

\theoremstyle{remark}

\usepackage[textsize=tiny]{todonotes}
\usepackage{array,etoolbox}
\preto\tabular{\setcounter{magicrownumbers}{0}}
\newcounter{magicrownumbers}
\newcommand\rownumber{\stepcounter{magicrownumbers}\arabic{magicrownumbers}}
\newcommand{\norm}[1]{\left\lVert#1\right\rVert}
\newcolumntype{K}{>{\centering\arraybackslash}m{.48in}}

\icmltitlerunning{How many perturbations break this model? Evaluating robustness beyond adversarial accuracy}

\begin{document}

\twocolumn[
\icmltitle{How Many Perturbations Break This Model? Evaluating Robustness Beyond Adversarial Accuracy}




\begin{icmlauthorlist}
\icmlauthor{Raphael Olivier}{yyy}
\icmlauthor{Bhiksha Raj}{yyy}
\end{icmlauthorlist}

\icmlaffiliation{yyy}{Language Technologies Institute, Carnegie Mellon University, Pittsburgh, USA}

\icmlcorrespondingauthor{Raphael Olivier}{rolivier@cs.cmu.edu}

\icmlkeywords{Machine Learning, ICML}

\vskip 0.3in
]



\printAffiliationsAndNotice{}  


\begin{abstract}

  Robustness to adversarial attacks is typically evaluated with adversarial accuracy.
  While essential, this metric does not capture all aspects of robustness and in particular leaves out the question of how many perturbations can be found for each point.
In this work, we introduce an alternative approach, adversarial sparsity, which quantifies how difficult it is to find a successful perturbation given both an input point and a constraint on the direction of the perturbation. 
We show that sparsity provides valuable insight into neural networks in multiple ways: for instance, it illustrates important differences between current state-of-the-art robust models them that accuracy analysis does not, and suggests approaches for improving their robustness. 
When applying “broken” defenses effective against weak attacks but not strong ones, sparsity can discriminate between the “totally ineffective” and the “partially effective” defenses. 
Finally, with sparsity we can measure increases in robustness that do not affect accuracy: we show for example that data augmentation can by itself increase adversarial robustness, without using adversarial training.
\end{abstract}

\section{Introduction}
\label{sec:intro}
\footnote{Our code is available at \url{https://github.com/RaphaelOlivier/sparsity}}Designing adversarially robust machine learning models has become one of the main objectives of the research community. Adversarial examples, these slightly perturbed inputs that pose significant problems for output prediction, are the source of multiple security threats: not only are they dangerous in and of themselves, but they also contribute to enabling data poisoning attacks \cite{Shafahi18} and even membership inference attacks \cite{pmlr-v139-choquette-choo21a}. Making models robust to adversarial perturbations is very difficult for multiple reasons; one of them is that the proper evaluation method of robustness is not a trivial problem.

A few years ago, it was common to evaluate defenses against a variety of adversarial attacks, like FGSM (\cite{goodfellow2014}), DeepFool (\cite{moosavi2016deepfool}) or JSMA (\cite{jsma}). Some defenses could claim good results on some attacks, but remained vulnerable to others (e.g. \citet{papernot15} broken in \citet{carlinidistillation}), suggesting multiple aspects to robustness and giving adversarial research the form of an ``arms race'' between attacks and defenses. Nowadays however, this approach has lost popularity and it is considered good practice to focus on the accepted strongest attacks for evaluation: typically PGD (\cite{madry18}) or its step-size free variant APGD \cite{croce2020reliable} for bounded attacks and Carlini\&Wagner (\cite{carlini16}) for unbounded attacks.

We argue that this approach ignores important aspects of adversarial robustness. PGD for instance is considered a surrogate to the worst-case accuracy given a fixed threat model. With input $x$, label $y$ and a set of admissible perturbations $\Delta$, worst case accuracy on $x$ is equal to 1 if:
\begin{equation}
\label{eqn:worstacc}
    \forall \delta \in \Delta,\; f(x+\delta)=y
\end{equation}
and 0 otherwise. This metric reflects whether, in the vicinity of an input point, there is \textit{one} successful perturbation. This can be misleading in the quest for robustness. Consider a hypothetical defense that, around every point, eliminates 99\% of all dangerous perturbations, leaving 1\% to find. Its worst-case accuracy would be $0\%$, just like an undefended model, and the defense would be considered ``broken'' by the research community. In other words, worst-case accuracy is biased towards defenses that are totally effective for some points, and against those that are partially effective around all points.
Whether a point has any or no adversarial perturbation is relevant knowledge, but it leaves out a lot of information. The ``broken'' defense above may for instance be useful in real-world contexts, where perturbations are harder to craft than on research datasets. Moreover, it may hypothetically be complementary to other ``broken'' defenses that eliminate a distinct subset of perturbations or improve a defense with non-zero adversarial accuracy.


\begin{figure*}
    \centering
    \includegraphics[width=0.3\linewidth]{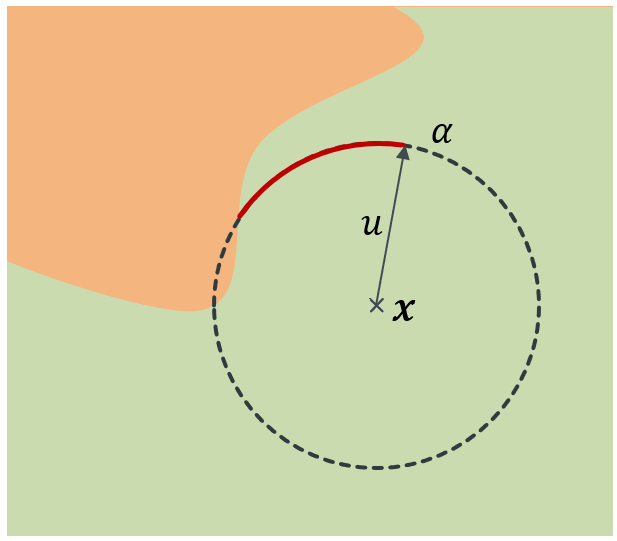}
     \hspace{0.1\linewidth}
    \includegraphics[width=0.3\linewidth]{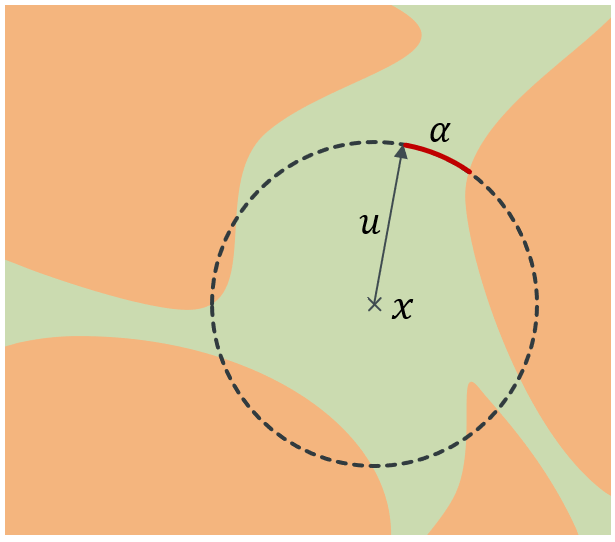}
    \caption{Representation of the decision boundary of some binary classification task (regions in green and orange) around point $x$ in two cases. In both cases, the model is not robust around $x$ with the represented $L_2$ radius $\epsilon$. However, there are many more adversarial perturbations on the right than on the left. Adversarial accuracy evaluates both cases at 0, but $\alpha$, the angle between fixed direction $u$ and the closest adversarial perturbation, is smaller on the right. Adversarial sparsity is the expected value of $\alpha$ when sampling $u$ uniformly.}
    \label{fig:example}
\end{figure*}

In this work, we propose an alternative approach for measuring adversarial robustness. Rather than measuring whether around an input $x$ there is \textit{at least one} adversarial perturbation, we try to estimate \textit{how many} such perturbations there are, i.e. the \textit{size} of the set of successful adversarial perturbations. In high-dimensional input spaces, there are not many informative and computationally tractable metrics to measure this size. Traditional measure theory hits considerable difficulties: for instance, within an $L_2$-ball of radius $\epsilon$ around $x$, the set of adversarial examples is way too small to be estimated with rejection sampling.

A naive approach would be to revert to evaluating with a set of more or less strong attacks and claim for instance that a model robust to FGSM but not PGD likely has fewer adversarial perturbations than a model robust to neither. However, there are good reasons why this approach was abandoned. Many defenses implicitly rely on gradient obfuscation effects \cite{athalye18} which makes attack convergence artificially difficult without actually defending against perturbations. Using only strong attacks, possibly enhanced with adaptive methods, protects evaluation against such effects. 

In our approach, rather than weakening the attack we propose to constrain it to only look at a subset of the set of admissible perturbations $\Delta$. The question we try to answer is: how large a typical subset must be to find an adversarial perturbation in it? We define metrics quantifying the expected size of the subset: the bigger it is, the fewer perturbations there are. We call this metric \textit{adversarial sparsity}.

A major challenge is to define a distribution of subsets and a metric adapted to the adversarial threat model. In section \ref{sec:metrics}, we define sparsity formally for $L_2$ and $L_\infty$ perturbations. In the $L_2$ case we consider uniform directions $u$ on the $L_2$ sphere, and adversarial sparsity is the expected angle $\widehat{\langle \delta,u \rangle}$ needed to find a successful perturbation. For $L_\infty$ attacks we consider vertices of the $L_\infty$ cube $u \in \{\pm1\}^n$, and measure the number of dimensions $k$ to modify to find a successful perturbation. We also discuss in section \ref{sec:metrics} how adversarial sparsity indeed reflects the size of the successful perturbations set.

Like worst-case accuracy, sparsity is not directly tractable, so we must estimate it using modified PGD attacks on a random sample of directions. Our algorithm is detailed in section \ref{sec:algs}. Armed with this tool we revisit multiple models and defenses proposed in prior works on the CIFAR10 dataset (section \ref{sec:exps}). We both include very strong defenses, taken among the strongest models on the RobustBench leaderboard \cite{croce2021robustbench}; much weaker defenses that are considered ``broken''; and undefended models trained with various data augmentation schemes. Using sparsity we discover or strengthen several properties of these models (section \ref{sec:res}), among which:

\begin{itemize}
    \item Higher adversarial accuracy does \textit{not} necessarily mean much fewer adversarial perturbations: many improvements on robustness benchmarks consist of removing a few residual perturbations around ``almost robust'' inputs.
    \item Most of the defenses ``broken'' by strong attacks do not reduce \textit{at all}  the number of perturbations, indicating that robustness claims on these defenses are spurious.
    \item Data augmentation methods, known to improve the results of adversarial training \cite{rebuffi2021fixing}, in fact increase robustness even \textit{without} adversarial training - though not enough to be reflected in adversarial accuracy.
    \end{itemize}

These findings could lead to promising developments in future robust models. They confirm the interest in alternative evaluation metrics for adversarial robustness.

\section{Related Work}
\label{sec:rel}
\subsection{Weak and strong attacks and defenses}
Multiple adversarial attacks have been proposed since the phenomenon was popularized by \citet{goodfellow13}. Early works include for instance the Fast Gradient-Sign Method (FGSM) attack \cite{goodfellow2014}. However, the two standard attacks for white-box defense evaluation in recent years are Projected Gradient Descent (PGD) \cite{madry18} for norm-bounded attacks, and Carlini\&Wagner (CW) \cite{carlini16} for regularization-based unbounded attacks. 

As it turns out, most defenses can be broken using these attacks, possibly in an adaptive fashion \cite{athalye18}. One recent improvement in attack success was provided by AutoAttack \cite{croce2020reliable}, an ensemble of attacks mostly based on PGD, which adaptively changes hyperparameters during optimization. Other recent attacks typically bring orthogonal improvement, such as perception-aligned attacks \cite{Gopfert19}.

Early work on adversarial attacks and defenses took the form of an arms race. Defenses breaking current state-of-the-art attacks were proposed, like adversarial training \cite{goodfellow2014}, input transformations \cite{xie2017mitigating,Guo2018} or projection to the data manifold \cite{defensegan} only to be broken soon after with a new attack or method \cite{athalye18}. For reasons mentioned above, we focus on the strong PGD attack in this work.

More recent works follow a somehow different trend. On the one hand, numerous proposed defenses are broken with attacks that already exist, for lack of rigorous evaluation. On the other hand, many works focus on promising approaches that have shown to be empirically or sometimes provably robust. Such approaches include randomized smoothing \cite{cohen19}, exact or relaxed certified methods \cite{kolter17}, and PGD-based adversarial training \cite{madry18,Wong20fast} which can be enhanced with data augmentation \cite{gowal2020uncovering,rebuffi2021fixing}. In this work, we explore both adversarial training and older, supposedly broken defenses, and revisit the extent of this non-robustness.

\subsection{Geometry of adversarial examples}
The idea to explore the spherical geometry of adversarial examples has some precedents in the literature. \citet{Tramr2017TheSO} take a linear algebra approach and estimate the dimension of a subspace of adversarial examples around a perturbation. \citet{Khoury2018OnTG}, and to an extent \citet{defensegan} explore the data manifold and attempt to explain adversarial examples as out-of-manifold points that models have trouble 
classifying. Such approaches that consider adversarial perturbations as artifacts have been challenged by works such as \citet{ilyas19} which demonstrate that adversarial perturbations are features that models can learn, and as such are reasonable input points.

We do not make claims about the ``nature'' of adversarial examples, but simply propose metrics to quantify them. This makes this work closer to \citet{Tramr2017TheSO} in its principle, although sparsity is different than dimension estimation. Parallels may also be drawn with Deep Hypersphere Embeddings \cite{Liu2017SphereFaceDH}, which have been shown to have some robustness properties \cite{Pang20}.

\subsection{Alternatives to adversarial accuracy}

The limitations of adversarial accuracy have led several previous works to propose alternative robustness metrics. We discuss those metrics and how they differ from ours.

\paragraph{Minimal perturbation:} it is common to compute the closest adversarial example to a given input as a metric for robustness around this point. Some attack algorithms are based on that approach \cite{carlini16}. This notion differs from sparsity in that it removes the notion of attack bound and perturbation constraint altogether: instead, all points are potential adversarial examples and the attacker finds the "best" one. A model could have a very small minimal perturbation but still have high sparsity for large radii.
\paragraph{Probabilistic robustness:} \citet{robey2022probabilistically} introduce a relaxation on "worst-case" robustness, by training models to be robust to most perturbations sampled uniformly, rather than all of them. Stemming from similar motivations to our work, probabilistic robustness differs from sparsity significantly in its approach, as it is "non-adversarial" robustness. Where we sample geometric regions to constrain the adversary, probabilistic robustness samples perturbations directly. Indeed, \citet{robey2022probabilistically} show that probabilistic robustness and adversarial robustness are largely uncorrelated, contrary to sparsity and accuracy.
\paragraph{Local Intrinsic Dimensionality (LID):} \citet{ma2018characterizing} analyze the local dimension of the submanifold of adversarial perturbations around an input. They use this metric to better understand the properties of adversarial regions. sparsity differs from LID as it does not assume that adversarial examples lie on a manifold. We probe regions and try to find adversarial examples in them, while LID starts from an adversarial example and probes the "adversarial region" around it. These approaches are intuitively complementary, in that sparsity estimates "how many" adversarial regions there are and LID "how big" those regions are. This duality is an interesting research topic for future work.

\section{Adversarial Sparsity}\label{sec:metrics}
\subsection{Definitions}
Throughout the following sections $f$ designates a machine learning model, and $L$ its loss function. $\epsilon$ is the radius of the attack, and $\Delta$ is the set of admissible perturbations. Usually $\Delta = B^n_k$ with $B^n_k$ the n-dimensional hyperball in norm $L_k$:
$$B^n_\infty = [0,1]^n$$
$$B^n_2 = \{x\in \mathbb{R}^n \mid x_1^2+...+x_n^2\leq1\}$$
where $x_j$ is the $i^{\text{th}}$ coordinate of $x$.

To make sparsity computations simpler, we will instead only consider the adversarial examples that use their entire perturbation budget. This means that for $k < \infty$ $\Delta= S^n_k$, using the hypersphere ($\norm{x}=1$) rather than the hyperball ($\norm{x}\leq1$). For $k=\infty$ $\Delta= \{\pm1\}^n$, i.e. we consider only the (finite) set of vertices of the hypercube. This is done with little cost of generality, as strong attacks like PGD nearly always use all of their perturbation budget (and the simpler FGSM attack is constrained to do so).

Given a point $x$ and a radius $\epsilon$, we define the \textit{adversarial set} as the set of successful admissible perturbations:
\begin{equation}\label{eqn:advset}
\text{Adv}(f,x,\epsilon) := \{\delta\in \Delta \mid  f(x+\epsilon.\delta) \neq f(x)\}
\end{equation}
In our convention, adversarial perturbations are always unit vectors, to be scaled by factor $\epsilon$ when applied. Finally, we can redefine adversarial accuracy over a point $x$ within this formalism:
\begin{equation} \label{eqn:acc}
\text{AA}(f,x,\epsilon) := \mathbf{1}[\text{Adv}(f,x,\epsilon) = \varnothing ]
\end{equation}

We will also sample points from probability distributions. $\mathcal{U}(X)$ designates the uniform distribution over measurable set $X$. 

\subsection{Defining sparsity}\label{sec:defspars}
Adversarial sparsity quantifies ``how big'' a subset of $\Delta$ typically needs to be to contain an adversarial perturbation. Formally, we assume access to a sequence of increasing subsets of $\Delta$: with $m_1<m_2$, $\emptyset \subseteq \Delta^{m_1} \subseteq \Delta^{m_2} \subseteq \Delta$. We can define adversarial sparsity relative to $\Delta^m$ as:
\begin{equation} \label{eqn:densu}
\begin{split}
\text{AS}(f,x,\epsilon, (\Delta^m)) := \inf\{m \mid \Delta^{m} \cap \text{Adv}(f,x,\epsilon) \neq \emptyset \}
\end{split}
\end{equation}
If we have a distribution $\mathcal{D}$ of such sequences, we can define adversarial sparsity as the expected value of $AS$
\begin{equation} \label{eqn:dens}
\overline{\text{AS}}(f,x,\epsilon) := \mathbb{E}_{(\Delta^m) \sim\mathcal{D}}[\text{AS}(f,x,\epsilon,(\Delta^m))]
\end{equation}

Intuitively, larger sparsity means more robust models, provided that some reasonable constraints are enforced on distribution $\mathcal{D}$. In particular, the distribution should be isotropic, i.e. not favor any particular direction.

We now give more concrete definitions relative to our two threat models. For $L_2$ perturbations we consider angular constraints. We sample $u\sim \mathcal{U}(S^n_2)$ and for $\alpha \in [0,\pi]$ we define $\Delta^\alpha$ as the spherical cap of direction $u$ and angle $\alpha$, that is the set of admissible perturbations that form an angle at most $\alpha$ with $u$:
\begin{equation}\label{eqn:cap}
\Delta^\alpha  = \text{Sc}(u,\alpha):=  \{\delta\in S_2^n \mid \delta \cdot u \geq \cos{\alpha}\}
\end{equation}

For $L_\infty$ perturbations, we sample both a vertex $u \in \mathcal{U}(\{\pm1\}^n)$ and a permutation of dimensions $\sigma \in \mathfrak{S}_n$. For $m \in \{1,...,n\}$ we define $\Delta^m$ as the set of perturbations differing of $u$ only over dimensions $\sigma(1),...,\sigma(m)$: 
\begin{equation}\label{eqn:cap}
\Delta^m :=  \{\delta\in S_\infty^n \mid \forall k>m \; \delta_{\sigma(k)} = u_{\sigma(k)} \}
\end{equation}

In other words $\Delta^m$ is the set of perturbations differing from $u$ by at most $m$ specific pixels $\sigma(1),...,\sigma(m)$. The permutation $\sigma$ is required to enforce that all dimensions are equally probable under the distribution.

\subsection{Sparsity and number of perturbations}\label{sec:sparsadvsize}

We propose sparsity as an approach to measure the size of the adversarial set. Is there indeed a relationship between the two? Intuitively this depends on how that set itself is distributed over $\Delta$. If for instance $\text{Adv}(f,x,\epsilon)$ is ``one half'' of $\Delta$ (e.g. all perturbations where the bottom left pixel perturbation is positive) then its sparsity will be 0 for half the directions but quite large for many others, leading to an expected sparsity that does not accurately reflect the immense size of this adversarial set. So which conditions should be met for sparsity to be a relevant metric?

Let us consider a simplified case where the set of perturbations is finite: $\text{Adv}(f,x,\epsilon) = \{\delta_1,...,\delta_k\}$. If we assume that the $\delta_i$ are uniformly sampled over the admissible set, then there are no preferred directions of high adversarial concentration, as opposed to the example above. In this case, the link between sparsity and $k$ can be mathematically established.
For $L_2$ perturbations, the expected value of $\text{AS}(f,x,\epsilon)$ when sampling $\delta_j$ is:

\begin{equation}\label{eqn:sparsityproof2}
\mathbb{E}[\text{AS}(f,x,\epsilon)]= \int_0^\frac{\pi}{2} ((1-t_\alpha)^k+(t_\alpha)^k)d\alpha
\end{equation}

with $t_\alpha = I_{sin^2(\alpha)}(\frac{n-1}{2},\frac{1}{2})$ where $I$ is the incomplete regularized beta function. Meanwhile for $L_\infty$ perturbations: 

\begin{equation}\label{eqn:sparsityproofinf}
\mathbb{E}[\text{AS}(f,x,\epsilon)]= \sum_{m=0}^n (1-2^{-m})^k
\end{equation}
In this case, we can also show that :
\begin{equation}\label{eqn:sparsityproofinf}
\frac{n-\log_2k}{4} \leq \mathbb{E}[\text{AS}(f,x,\epsilon)] \leq n-\log_2k + \frac{e}{e-1}
\end{equation}

i.e. sparsity tends to vary like $n-\log_2k$. We defer all proofs to appendix \ref{apx:math}.

In the general case $\text{Adv}(f,x,\epsilon)$ is infinite\footnote{With our definition of the admissible set, $\text{Adv}(f,x,\epsilon)$ is finite in the $L_\infty$ case. But the same reasoning applies if we consider the more general set $\Delta = S^n_\infty$} and we are interested in its volume relative to $\Delta$, rather than its cardinal. Both situations can be linked if considering for instance 
\begin{equation}
\text{Adv}(f,x,\epsilon) = \bigcup_{j=1}^k \text{Sc}(\delta_j,\beta)
\end{equation}
ie $\text{Adv}(f,x,\epsilon)$ is the union of spherical caps of equal radius $\beta$. $\beta$ can be called the ``local adversarial radius'', i.e. how much a typical adversarial perturbation should be changed to recover the original input. In that case, the volume of $\text{Adv}(f,x,\epsilon)$ is equal to $k.V(\text{Sc}(u,\beta))$ (assuming disjoint union) and $\text{AS}(f,x,\epsilon)$ would vary by at most $\beta$ compared to the finite case. In other words, sparsity depends on both $k$ and $\beta$. Therefore comparing the sparsity of two models is akin to comparing the size of their adversarial sets provided that these sets have similar local radii. We implicitly make that assumption in our experiments: estimating the local adversarial radius as well as sparsity would be an interesting extension of this work.

\section{Algorithms}\label{sec:algs}
In this section, we detail the algorithms we use to empirically compute sparsity. We first describe the Projected Gradient Descent (PGD) attack, and the modifications we implement to compute adversarial examples in constrained regions. Then we describe the full sparsity computation method, which applies PGD with multiple constraints. We sum up the full procedure in $L_2$ norm in Algorithm \ref{alg:sparsity}.

\subsection{Projected Gradient Descent}\label{sec:pgd}

The PGD attack \cite{madry18} is a strong first-order attack, i.e. it only uses model gradient information. It optimizes the following non-convex objective: 
\begin{equation}\label{eqn:pgd}
    \arg\max_{\norm{\delta}<\epsilon}L(f(x+\delta),y)
\end{equation}
using projected gradient descent for a number $n$ of gradient update steps. At each step $k+1$ consists of a gradient optimization step followed by a projection step which depends on the attack norm:
\begin{equation}\label{eqn:pgditer}
\delta'_k \leftarrow \delta_{k} + \eta.\text{sign}(\nabla_{\delta_{k}}  L(f(x+\delta_{k}),y)) \;\;\text{(grad. ascent)}
\end{equation}
\begin{equation}\label{eqn:pgdprojlinf}
\delta_{k+1} \leftarrow \text{clip}(\delta'_k,\epsilon) \;\;\text{($L_\infty$ projection)}
\end{equation}
\begin{equation}\label{eqn:pgdprojl2}
\delta_{k+1} \leftarrow \min(1,\frac{\epsilon}{\norm{\delta'_k}_2}).\delta'_k \;\;\text{($L_2$ projection)}
\end{equation}


\subsection{Constrained PGD}\label{sec:angle}
To practically compute angular sparsity we need to design a constrained attack, able to project not on a hyperball but on the subsets defined in section \ref{sec:defspars}.

For $L_\infty$ attacks this is straightforward: given a vertex $u$, a perturbation $\sigma$, and a number of dimensions $m$, we append after projection an additional step enforcing the constraints:
\begin{equation}\label{eqn:infpgd1}
\forall k>m\;\; \delta_{\sigma(k)} \leftarrow u_{\sigma(k)}
\end{equation}

For $L_2$ attacks, the steps are slightly more complex, as we need an angular projection on the spherical cap of angle $\alpha$ and direction $u$. This requires replacing the projection step in equation \ref{eqn:pgdprojl2} with a projection by both norm and angle. The technical steps of that projection are detailed in Appendix \ref{apx:proj}

We name $\text{PGD}_m$ and $\text{PGD}_\alpha$ these constrained attacks in the following.

\subsection{Computing sparsity over a point}
To estimate sparsity for $L_2$ (resp. $L_\infty$) attacks, given a direction $u$ (resp. $u,\sigma$) we explore possible values of $\alpha$ (resp. $m$) with binary search, and at each step run the constrained attack. We average over multiple sampled directions to estimate $\overline{\text{AS}}$ (in general 100). 

The equivalent $L_\infty$ algorithm is similar, replacing only direction $u$ by $(u,\sigma)$, angle $\alpha$ by a number of pixels $m$ and the constrained $L_2$-PGD by the constrained $L_\infty$ PGD. The time complexity of these algorithms is discussed in Appendix \ref{apx:timecplx}.
\begin{algorithm}
\caption{$L_2$ sparsity computation algorithm\label{alg:sparsity}}
\begin{algorithmic}
\Require model $f$, point $(x,y)$, $K\in \mathbb{N},N \in \mathbb{N}$
\State $k\gets 0$
\State $i \gets 0$
\While{$i \leq N$}
\State $\alpha_0^i \gets 0$, $\alpha_1^i \gets \pi$, $u_i \sim \mathcal{U}(S^n)$
\State $i \gets i+1$
\EndWhile
\While{$k\leq K$}
\State $i \gets 0$
\While{$i \leq N$}
\State $\alpha^i \gets \frac{\alpha_0^ i+\alpha_1^ i}{2}$
\State $x_{\text{adv}} \gets \text{PGD}_{\alpha^i}(f,x,y,u_i)$
\If{$f(x_{\text{adv}}) \neq y$}
    \State $\alpha_1^i \gets \alpha^i$
\Else
    \State $\alpha_0^i \gets \alpha^i$
\EndIf
\State $i \gets i+1$
\EndWhile
\State $k \gets k+1$
\EndWhile
\State \Return $\frac{1}{n}\sum_0^n\alpha^i$
\end{algorithmic}
\end{algorithm}

\subsection{Computing sparsity over a dataset}\label{sec:evalrob}
Adversarial sparsity only makes sense if there \textit{are} adversarial perturbations in the vicinity of an input $x$. It is designed to capture the level of vulnerability in non-robust models. We choose to restrict sparsity computation to the residual subset of inputs where the model is vulnerable. For instance, if a model is robust over $50$ inputs and has sparsity 0.3 over another $50$, we will say it reaches 50\% accuracy and \textit{residual sparsity} 0.3. An alternative could be to default sparsity to a max value when evaluating a model that is robust around $x$ ($\text{Adv}(f,x,\epsilon)=\varnothing$): $\pi$ for $L_2$ attacks, and dimension $n$ for $L_\infty$ attacks. 
\begin{table*}[h!]
\centering
 \begin{tabular}{|@{\makebox[1.5em][r]{\rownumber\space}} |c | c | KKK| KKK|} 
 \hline
 Model & Defenses & Clean acc. & Adv. acc. & Sparsity & Clean acc. & Adv. acc. & Sparsity    \\ [0.5ex] 
 \hline
 
R-18 & None & 88  & 0	& 0.180	 &  88  & 0	 &  56.8 \\
 \hline
 
R-18 & \citet{madry18} (1 step) &  90 & 	62.1 & 	0.559 &   	90.73 & 	40.08 & 	229 \\

R-18 & \citet{madry18} & 88.8	 & 67.3 & 	0.553 & 	80.97 & 	47.19 & 277	 \\

R-18 & \citet{madry18} (w/o aug.) & 82.8 &	60.2 &	0.532 & 73.96 & 	43.02 & 273	 \\

WR-70-16 & \citet{rebuffi2021fixing} (w/ data) & 95.74 & 	82.3 & 	0.581 & 	94.16 &	62.91	 & 213  \\

WR-70-16 & \citet{gowal2020uncovering} (w/ data) & 94.74 & 	80.5 & 	0.590 & 	91.17 & 	62.5 & 	215 \\

WR-70-16 &  \citet{rebuffi2021fixing} & 92.41 & 	80.4 & 	0.545 & 	89.16 & 	65.83 & 	202 \\

WR-28-10 &  \citet{rebuffi2021fixing} & 91.79 & 	78.8 & 	0.529 & 	89.58 & 	61.66 & 	209 \\

R-18 & \citet{Sehwag2021Proxy} & 89.5 & 	73.4 & 	0.539 & 	87.08 & 	52.08 & 	231 \\

R-18 &  \citet{rade2021helperbased} & 90.5 & 	76.2 & 	0.515 & 	92.08 & 	57.5 & 	209 \\

WR-34-10 & \citet{Augustin2020AdversarialRO} & 92.23	 & 76.3 & 	0.525 & 	- & 	-	 & - \\

R-50 & \citet{Augustin2020AdversarialRO} & 91.08 & 	72.9 & 	0.572 & 	- & 	- & 	- \\

WR-34-R & \citet{huang2021exploring} & - & 	- & 	- & 	89.58 & 	57.5	 & 227 \\

WR-34-R & \citet{huang2021exploring} (EMA) & - & 	- & 	- & 	89.17 & 	57.5	 & 230 \\

 \hline
 \end{tabular}
 \caption{Evaluation of state-of-the-art adversarially trained models under $L_2$ ($\epsilon=0.5$) and $L_\infty$ ($\epsilon=0.03$) perturbations. We report Natural accuracy (without attack), adversarial accuracy (under AutoPGD), and adversarial (residual) sparsity. Model architectures are either ResNet (R) or WideResNet (WR) . (w/o aug.) means the model was trained without any data augmentation: most models are trained with at least Crop+Resize augmentation. (w/ data) means that external data was used to train the model. For some defenses there is only a $L_\infty$ or only a $L_2$-trained model available; for many both exist, in which case we report the results of both. \label{tab:strongres}}
\end{table*}

\begin{table}[h!]
\centering
 \begin{tabular}{ |c | c | c | c|} 
 \hline
 Defense & Test Accuracy & $L_2$ Sparsity & $L_\infty$ Sparsity    \\ [0.5ex] 
 \hline
 None	&	88 & 0.180 & 56.8 \\
 \hline
 JPEG	&	77.1 &	0.161 &	80.9 \\
FS	&	87.8 &	0.178 &	60.1 \\
SPS	&	74 &	0.147 &	77.7 \\
FS+SPS	&	74.5&	0.174&	86.6 \\
\hline
 \end{tabular}
 \caption{Evaluation of ResNet18  with various ``broken'' defenses under $L_2$ attack with 20 iterations and $\epsilon=0.5$. We report Natural accuracy (without attack) and angular sparsity for $L_2$ ($\epsilon=0.5$) and $L_\infty$ ($\epsilon=0.03$) perturbations. \label{tab:weakres}}
\end{table}

\begin{table}[h!]
\centering
 \begin{tabular}{ |c | c | c | c|} 
 \hline
 Augmentation & Accuracy & $L_2$ Sparsity & $L_\infty$ Sparsity  \\ [0.5ex] 
 \hline
 None	&	88.0&	0.180&	56.8 \\
  \hline
Crop+Resize	&	93.7&	0.225 &	64.0 \\
 Cutmix	 & 	94.58	 & 0.185  & 50.8 \\
50\%	Cutmix  & 	95.83 & 	0.202 & 	57.0 \\
Mixup	 & 	93.75	 & 0.157 & 	71.1 \\
50\%	Mixup   & 	92.5 & 	0.202 & 	69.4 \\
Cutout	 & 	94.16 & 	0.225 & 	62.2\\
50\% Cutout  	 & 	93.75 & 	0.225 & 	62.6 \\
Ricap	 & 	92.2 & 	0.272 & 	105 \\
50\% Ricap  	 & 	91.67 & 	0.253 & 	83.6 \\
\hline
 \end{tabular}
 \caption{Evaluation of ResNet18  with data augmentation schemes. The augmentation is applied either on all inputs or 50\% of inputs \label{tab:augres}}
\end{table}

\section{Experiments} \label{sec:exps}
We run experiments in $L_2$ and $L_\infty$ perturbations on the CIFAR10 dataset \cite{cifar10}. Additional results on ImageNet can be found in Appendix \ref{apx:addexps}.
We use attack radii $\epsilon=0.5$ and $\epsilon=8/255$ respectively.

\subsection{Models and defenses}
We mainly evaluate defenses over a classic residual CNN architecture: ResNet-18 \cite{resnet}. We evaluate larger ResNets and WideResNets as well. We consider multiple defense mechanisms:

\subsubsection{Adversarial training} Adversarial training consists in applying adversarially perturbed inputs during training rather than or along with natural inputs. It is one of the most robust defenses against PGD attacks. 
We use multiple pretrained models using state-of-the-art defenses based on adversarial training \cite{gowal2020uncovering,Augustin2020AdversarialRO,rebuffi2021fixing,rade2021helperbased, huang2021exploring}. These defenses are well ranked in the Robustbench leaderboard \cite{croce2021robustbench} for $L_2$ attacks on CIFAR10. Some use additional extra data for pretraining. We also train ResNet-18 models with PGD training, following \cite{madry18}, with 1 or 10 attack steps.

\subsubsection{Preprocessing}
Early defenses often used input preprocessing during testing to ``erase'' adversarial noise. These defenses have mostly been beaten by stronger or adaptive attacks. We revisit some of them: JPEG compression and decompression \cite{Guo2018}, Feature Squeezing, and Spatial Smoothing \cite{Xu18fs}. The latter two were proposed as complementary defenses. This makes analysis with sparsity particularly interesting to determine whether each defense has a partial effect on robustness and would benefit from ensembling.

\subsubsection{Data augmentation} \cite{xie2017mitigating} suggest that random transformations can mitigate adversarial attacks.  This defense is also considered broken when using strong attacks \cite{athalye18}. Additionally, \cite{rebuffi2021fixing} showed that data augmentation can significantly improve the performance of adversarial training. 

We wonder if data augmentation affects robustness by itself. We train multiple models with standard training and data augmentation schemes used in \cite{xie2017mitigating} and \cite{rebuffi2021fixing}, and evaluate them with sparsity.

\subsection{Attack}

Preprocessing methods (JPEG, Feature Squeezing, Spatial Smoothing) may pose obfuscation problems when computing adversarial attacks during sparsity estimation. Hence, we follow \cite{shin2017jpeg} and use Differentiable JPEG, which we backpropagate through. We use the Straight-Through estimator \cite{athalye18} against the other two.

In addition to sparsity, we evaluate adversarially trained models using the AutoAttack \cite{croce2020reliable} and $\epsilon=0.5$ or $\epsilon=8/255$. To reduce computation time we run a slightly cheaper AutoAttack,  using only APGD-CE and APGD-DLR. All other models (weakly defended or undefended) achieve 0\% adversarial accuracy, which we do not report in result tables.

\section{Results}\label{sec:res}
The results of adversarially defended models are reported in Table \ref{tab:strongres}. We report averaged values of sparsity over the first 1000 vulnerable inputs in the CIFAR10 test set and, for each input, 100 random directions. 

\subsection{Margin of error when estimating sparsity}\label{sec:anadir}
Angular sparsity $\overline{\text{AS}}$ cannot be directly computed but must be approximated with the sample mean estimator over multiple directions. To assess the margin of error in this estimation we compute for multiple models, the standard deviation of $L_2$ sparsity over 100 directions, for 1000 input points. We find that these deviations are consistently lower than 0.022. Using the general formula for margins of errors $z*\frac{\sigma}{\sqrt{n}}$,
we conclude that our estimates with 100 directions provide estimates within a \textbf{$\pm 0.002$} margin with $95\%$ confidence. In Appendix \ref{sec:anapoint} we provide additional information on the sparsity variance, this time with respect to input points.

\subsection{Comparing strong defenses by sparsity}\label{sec:strong}

At first sight, we can already observe in Table \ref{tab:strongres} that sparsity is overall consistent with adversarial accuracy. Adversarially defended models, whose accuracy is above 40\%, have much higher sparsity than the undefended baseline. Models trained with strong adversarial training (lines 3-15) all reach $L_2$ (resp $L_\infty$) sparsity greater than 0.5 (resp 200). In comparison, the baseline model achieves 0.180 (resp. 56.8). 

However, among the robust models (lines 3-15), sparsity and accuracy variations are not well correlated. In fact, when it comes to $L_\infty$ perturbations, the model with higher accuracy has often \textit{lower} sparsity on the residual subset! We illustrate this phenomenon by plotting for robust models sparsity as a function of accuracy in Figure \ref{fig:scatter}. An explanation would be that improvements in adversarial training from a baseline model A focus on robustly classifying the easiest points, i.e. those for which few perturbations fool A. Those points have higher sparsity - thus classifying them well drops the residual sparsity for the remaining points.

\begin{figure}[t]
    \includegraphics[width=0.5\textwidth]{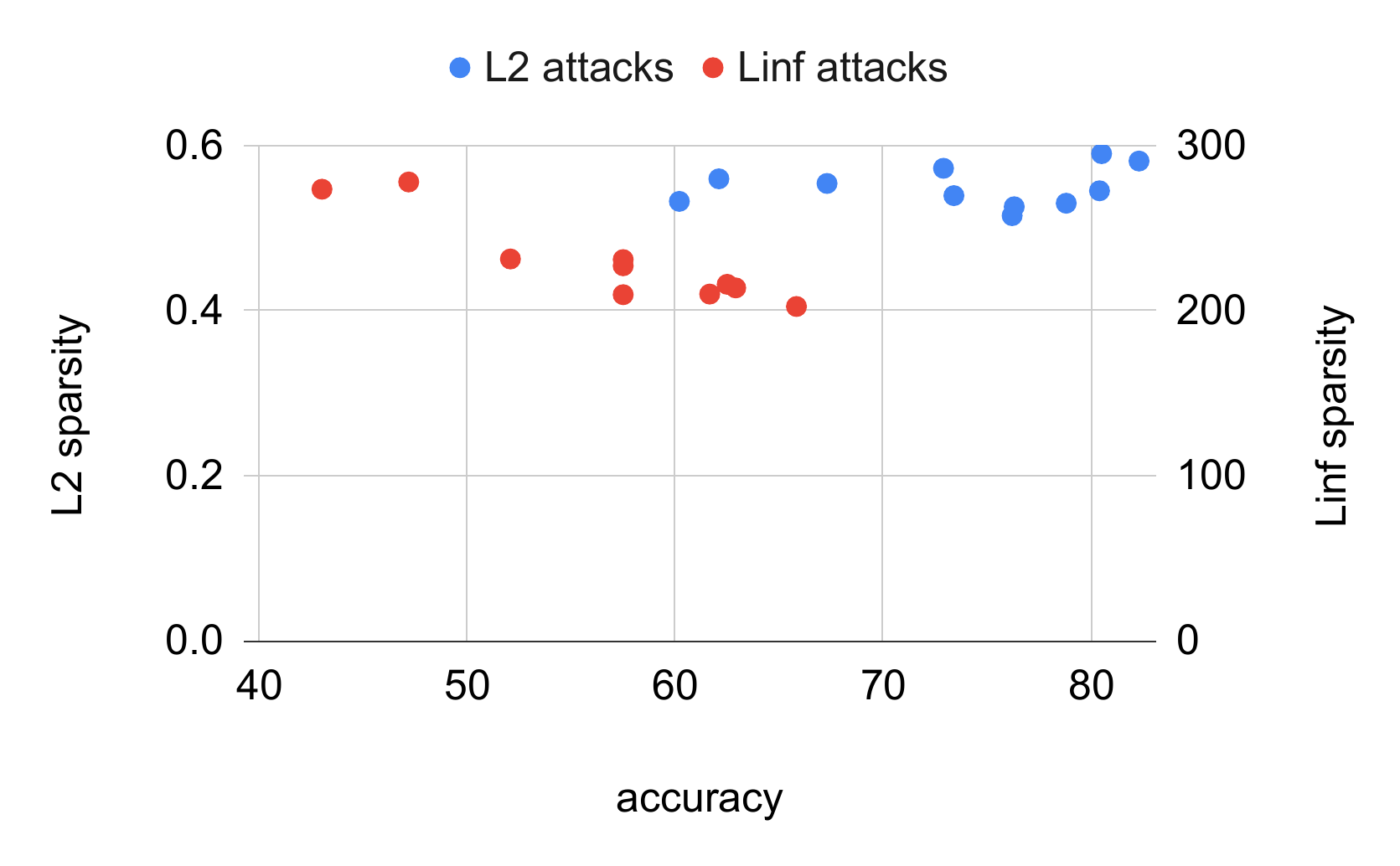}%
    \caption{Sparsity as a function of adversarial accuracy for all robust models, in both $L_2$ norm (blue) and $L_\infty$ norm (red). The values used are the same as in Table \ref{tab:strongres} \label{fig:scatter}}%
\end{figure}%


\begin{figure*}
     \centering
     \begin{subfigure}[t]{0.49\textwidth}
         \centering
         \includegraphics[width=\textwidth]{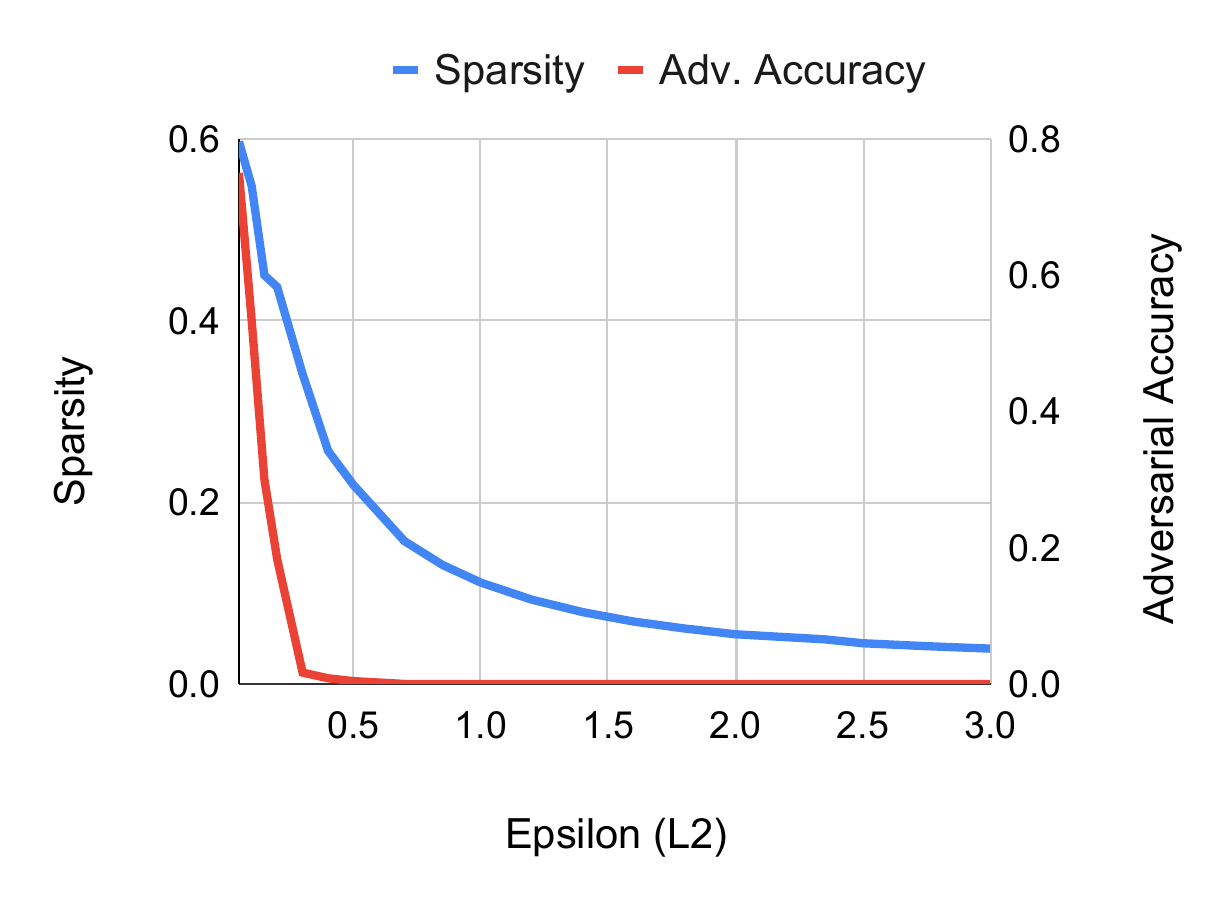}
         \caption{Standard ResNet18 model \label{fig:epsvaryl2}}
     \end{subfigure}
     \hfill
     \begin{subfigure}[t]{0.49\textwidth}
         \centering
         \includegraphics[width=\textwidth]{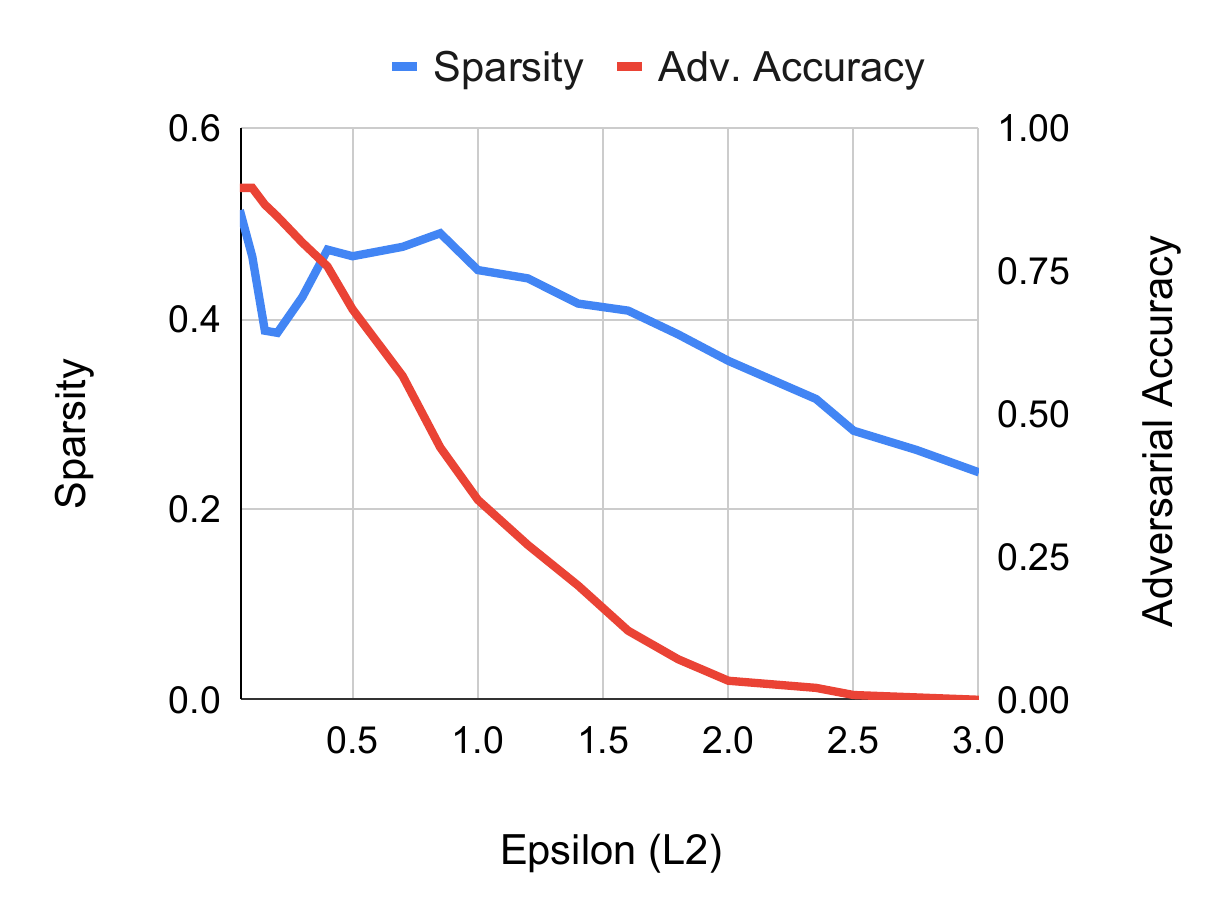}
         \caption{Adversarially trained model (training radius $0.5$) \label{fig:epsvaryl2adv}}
     \end{subfigure}
        \caption{Evolution of $L_2$ sparsity and adversarial accuracy of ResNet18 models as a function of the attack radius $\epsilon$ }
\end{figure*}

\subsection{Comparing broken defenses by sparsity}\label{sec:broken}
In Table \ref{tab:weakres} we compare the sparsity of the vanilla model and models defended with one or more of the ``broken'' preprocessing-based defenses. Using the strong AutoAttack attack rather than the attacks these defenses' authors had used for evaluation, we can break them on all inputs ($0\%$ adversarial accuracy) for both $L_2$ and $L_\infty$ attacks. This weakness to strong attackers was, for all defenses, first pointed out in the articles that successfully broke them \cite{athalye18,He17ensemble}.

The fact that some defenses are effective against weak attacks like FGSM but not against stronger attacks has two possible explanations. One is that such defenses protect against a subset of all existing perturbations, but not all of them: a stronger attacker can find more subtle perturbations that evade the defense. Another is that defenses don't actually protect against perturbations at all, but merely make them perturbations harder to find for attackers by \textit{obfuscating gradients} \cite{athalye18}. 

Adversarial sparsity offers a simple way to discriminate between these two situations. We expect a partially effective defense to increase sparsity, but not accuracy. An obfuscation-based defense on the other hand would not improve either of these metrics.

On $L_2$ attacks, in Table \ref{tab:weakres} we observe that none of these defenses lead to a significant improvement in sparsity. Any $L_2$-robustness claim regarding these methods is therefore likely to be relying on spurious obfuscation effects. Results are however different on $L_\infty$ attack. While Feature Squeezing does not increase sparsity, JPEG compression and spatial smoothing do ($+20.9$ and $+18.7$ respectively). These defenses are therefore effective against some perturbations. Moreover, feature smoothing does boost robustness when combined with Spatial Smoothing, which was one of the claims of \citet{Xu18fs}. 
We argue that the recent claims that most proposed defenses rely solely on subtle obfuscation effects should be revised. In a few examples, we have shown that these defenses offer non-negligible protection against some perturbations. Even if past works have identified obfuscation effects in these defenses using different methods, such as adaptive evaluation \cite{athalye18}, they do not explain all of their effect on robustness. Sparsity is complementary to adaptive attacks when evaluating defenses.

\subsection{Testing data augmentation on standard models}\label{sec:ensemble}

In Table \ref{tab:augres} we report the adversarial sparsity of various models trained with data augmentation. We observe that each of them increases sparsity on both $L_2$ and $L_\infty$ perturbations. The only exception is CutMix \cite{cutmix}. 

Interestingly CutMix is also the best augmentation for robustness when combined with adversarial training, according to \citet{rebuffi2021fixing}. Moreover, RICAP, which largely outperforms all other augmentations for robustness with standard training in our experiments, is the worst-performing one for adversarial training! A possible explanation is that despite improving robustness, data augmentation interferes with adversarial training: the best augmentations for robustness won't necessarily combine harmoniously with the best training schemes. This explanation is consistent with past works which had shown that augmentation leads to no robustness improvement \cite{Rice2020OverfittingIA}. One of the contributions of \citet{rebuffi2021fixing} was to overcome this issue with methods like weight averaging. 

Our results suggest that the potential of RICAP augmentation for robustness is still underexplored. Given its considerable effect on the distribution of adversarial examples, it is likely that a version of RICAP+Adversarial training could outperform CutMix+Adversarial training if the aforementioned interference effects are overcome.

\subsection{Influence of the attack radius}

Another interesting application of sparsity is to provide a smoother robustness metric than accuracy. We illustrate this in Figure \ref{fig:epsvaryl2}, where we plot both sparsity and accuracy for $L_2$ attacks for increasing values of the attack radius $\epsilon$, and a standard ResNet18. We observe that while accuracy drops to zero after a small value of $\epsilon$, sparsity decreases much more slowly, only converging to 0 asymptotically.

This property lets us compare models on a much larger range of perturbations. In Figure \ref{fig:epsvaryl2adv} we plot the equivalent plot for an adversarially trained ResNet18. When comparing Figures 2a and 2b, on large $\epsilon$ values accuracy provides little information, being null or close to null for both models. On the other hand, we easily observe that sparsity is over five times higher for the adversarially trained model. This shows that training models adversarially with a given attack radius is beneficial even for much larger radii.

In Figure \ref{fig:epsvaryl2adv} we also observe an interesting fluctuation in the sparsity curve. Although sparsity overall decreases when the attack radius increases overall, the training radius $\epsilon=0.5$ appears to be a local sparsity maximum. One interpretation is that for some points, the effect of adversarial training is to remove adversarial perturbations from the unit hypersphere, but not inside it. This leads to a strange effect where the model gets a bit more vulnerable when decreasing the attack radius. Because adversarial perturbations in sparsity are restricted to points on the hypersphere, it is particularly affected. 

In Appendix \ref{apx:linfradius} we provide equivalent plots for $L_\infty$ attacks and $L_\infty$ adversarial training, with very similar observations.

\section{Conclusion}
We have proposed a novel robustness metric named adversarial sparsity. We have shown that it is complementary to adversarial accuracy, offering additional insight into both weak and strong defenses. By applying it to data-augmented models we have found evidence suggesting that some augmentation methods still retain an untapped potential to increase robustness. 

While accuracy (certified or not) should likely remain the primary metric for benchmarking strong defenses, we believe that using finer metrics in the research process can benefit the research field. An important direction for future work could be to extend sparsity to threat models beyond norm-bounded perturbations, such as human-perception-based attacks or common corruptions.

\bibliography{example_paper}
\bibliographystyle{icml2023}

\newpage
\appendix

\section{Projection on a spherical cap}\label{apx:proj}
The computation of $L_2$ sparsity requires us to replace the projection step in PGD with a projection on a spherical cap (section \ref{sec:angle}). This projection consists of the following steps:
\begin{equation}\label{eqn:apgd1}
p \leftarrow \delta'_k - \delta'_k\cdot u . u
\end{equation}
\begin{equation}\label{eqn:apgd2}
r \leftarrow min(\frac{\norm{\delta'_k}_2}{\norm{p}_2}\sin{\alpha},1)
\end{equation}
\begin{equation}\label{eqn:apgd3}
s \leftarrow \delta'_k\cdot u . u + r.p
\end{equation}
\begin{equation}\label{eqn:apgd4}
\delta_{k+1} \leftarrow \frac{\min(\epsilon,\norm{\delta'_k}_2)}{\norm{s}_2}s
\end{equation}

The output of Eq \ref{eqn:apgd3} is a linear combination of $\delta'_k$ and $u$ (with positive coefficients) whose angle with $u$ is $\min(\widehat{\delta'_k,u},\alpha)$. We then rescale this vector as in standard PGD to obtain $\delta_{k+1}$. In Figure \ref{fig:angle} we visually illustrate how this procedure returns the closest perturbation to $\delta'_k$ of angle at most $\alpha$ with $u$.

\begin{figure}[t]
    \includegraphics[width=0.4\textwidth]{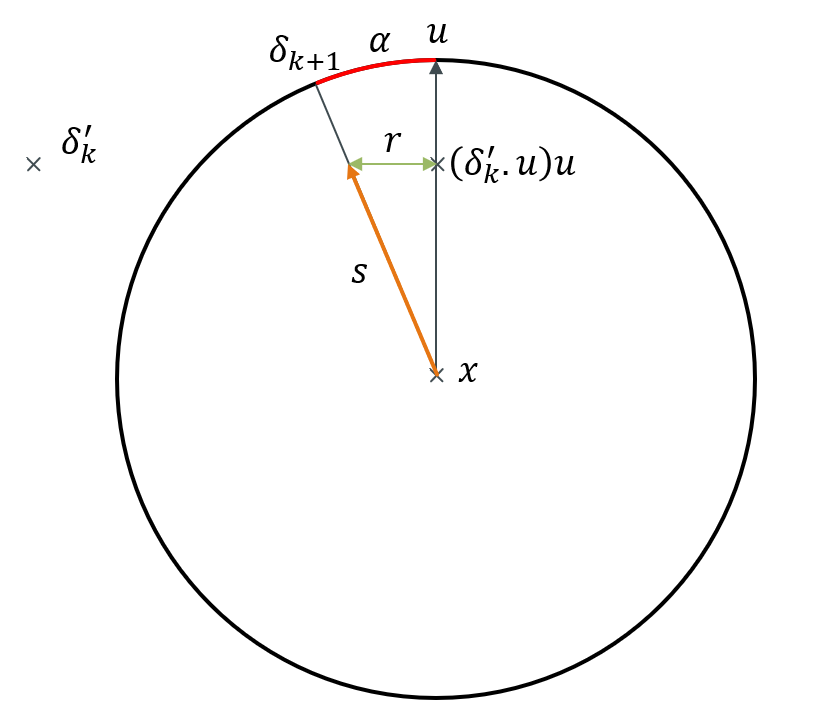}%
    \caption{Illustration of the angular projection steps on a 2-dimensional sphere, with $\epsilon=1$. $\delta'_k$ is the perturbation before projection, $u$ is the direction of the cap, and $\delta_{k+1}$ is the final perturbation, projected both in norm and angle. \label{fig:angle}}%
\end{figure}%
\section{Time complexity of sparsity computation}\label{apx:timecplx}
On CIFAR10, computing adversarial sparsity around an input point with 100 directions, 10 search steps, and 20 PGD iterations takes a few seconds for a ResNet-18 model on an Nvidia RTX 2080 Ti. For reference, we find it shorter than running the AutoPGD attack \cite{croce2020reliable} with default hyperparameters in its worst-case scenario (i.e. when there is no adversarial example).

\subsection{Dependence on input dimension}
Everything else equal, the duration of sparsity computation in a given direction is proportionate to the duration of one backward pass on the model. It is therefore not significantly longer on ImageNet than on CIFAR10 when the same model architecture is used, the only difference being the input size. Of course, reaching good performance on ImageNet requires (to this day) larger models than on CIFAR10: the current state-of-the-art models have about 2B parameters \cite{Dosovitskiy2021AnII}, while our largest models in this work have less than 100M. This would impact the practical cost of sparsity computation. 

One could also think that the number of samples required to confidently evaluate sparsity is much larger in higher dimensions. However, our algorithm does not require to sample in ``every'' direction, or enough to break the ``curse of dimensionality''. Sparsity in practice has low variance with respect to the direction (see Section \ref{sec:anadir}). Therefore only a few samples are required to estimate with a reasonable margin of error.

\subsection{Improving the complexity of sparsity}
While attacks over multiple directions can be computed in batches, binary search constitutes the major bottleneck of this algorithm. Some heuristics could help speed up sparsity computation. For example, we could use batched n-Ary search with fewer directions to find a first approximation of sparsity, then confirm/refine it with more directions. As we have mostly worked with reasonably sized models and inputs we have not experimented with such heuristics and leave them as future work.
\begin{table*}[t]
\centering
 \begin{tabular}{|@{\makebox[1.5em][r]{\rownumber\space}} |c | c | c c c|} 
 \hline
 Architecture & Defenses & Natural accuracy & Adv. accuracy & Sparsity    \\ [0.5ex] 
 \hline
 ResNet-50 & None & 87.8\% & 0\% & 0.194 \\ [0.5ex] 
 ResNet-18 & \cite{Salman2020} & 52.9\% & 39.5\% & 0.426 \\ [0.5ex] 
 ResNet-50 & \cite{Salman2020} & 64.0\% & 53.4\% & 0.415 \\ [0.5ex] 
 \hline
 \end{tabular}
 \caption{Evaluation of defended and undefended ImageNet models under $L_2$ attack with 20 iterations and $\epsilon=1.0$. We report Natural accuracy (without attack), adversarial accuracy (AutoPGD), and adversarial angular (residual) sparsity. \label{tab:imgnet}}
\end{table*}

\section{Additional experiments}\label{apx:addexps}
\subsection{ImageNet}
$L_2$ attacks on ImageNet are not as popular as on CIFAR10. We however apply them, both for reference and to verify that angular sparsity computation scales reasonably to a larger input size. We evaluate pretrained ResNet models provided in the Robustbench framework. One is trained in a standard fashion and the other adversarially against $L_\infty$ perturbations, following \cite{Salman2020}. All take inputs of size 224x224. We apply $L_2$ perturbations of radius $\epsilon=1.0$. Everything else is similar to the CIFAR10 experiments.

We report the results in Figure \ref{tab:imgnet}. We report higher robustness (under all metrics) than we did for CIFAR10; this is consistent with the fact that $\epsilon=1.0$ in a 224x224 vector space allows a smaller perturbation per pixel budget than $\epsilon=0.5$ in a 32x32 space. In the absence of $L_2$-robust, easily available ImageNet models we chose a small value of $\epsilon$. 

\subsection{Visual Transformers}
We evaluate a B-16 Visual Transformer Architecture (ViT), pretrained on ImageNet and fine-tuned on CIFAR10 \cite{Dosovitskiy2021AnII}. Recent works have shown that these models learn different features than Convolutional Neural Networks \cite{Raghu2021} which raises the question of their behavior against adversarial examples. In fact, some works have already claimed that ViTs are more robust to attacks than CNNs \cite{Shao21}. Investigating these claims is an interesting use case of adversarial sparsity.

We evaluate this ViT model on $L_2$ perturbations with the same parameters as in Section \ref{sec:res}. It reaches a natural accuracy of 97\%, greater than any of our ResNet-18 models, and an adversarial accuracy of 0\%. Its $L_2$ sparsity is 0.183, almost equal to that of the ResNet-18 model trained with data augmentation (line 3). The training recipe we use to fine-tune the model employs similar augmentation methods. This suggests that ViTs and ResNets behave similarly against the same $L_2$ PGD threat model. This challenges the conclusions of \cite{Shao21} on the adversarial robustness of ViTs. 

\subsection{Point-wise variance}\label{sec:anapoint}
When varying the input point there are significant variations in the value of the metrics. Taking the vanilla model evaluated over 100 inputs, for $L_2$  (resp. $L_\infty$) sparsity we observe a standard deviation of $0.144$ (resp. $52.8$) with respect to data points. Keeping in mind the considerations in Section \ref{sec:sparsadvsize}, this hints that the adversarial set is considerably larger for some points than others. The deviation is even larger for adversarially trained models. This is clearly visible in Figure \ref{fig:histogram}, where we plot the histogram of sparsity values of standard and adversarially trained ResNet18 models.

\begin{figure*}
     \centering
     \begin{subfigure}{0.9\linewidth}
         \centering
         \includegraphics[width=\linewidth]{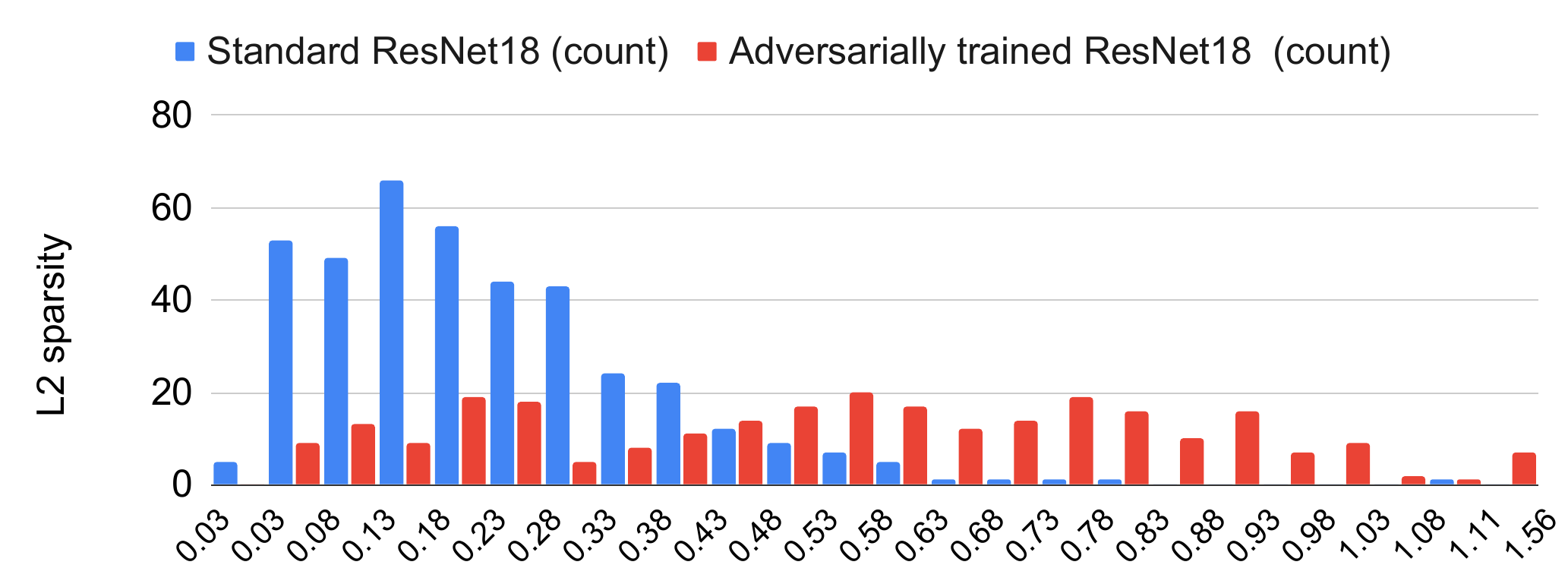}
         \caption{$L_2$ sparsity \label{fig:histl2}}
     \end{subfigure}
     \hfill
     \begin{subfigure}{0.9\linewidth}
         \centering
         \includegraphics[width=\linewidth]{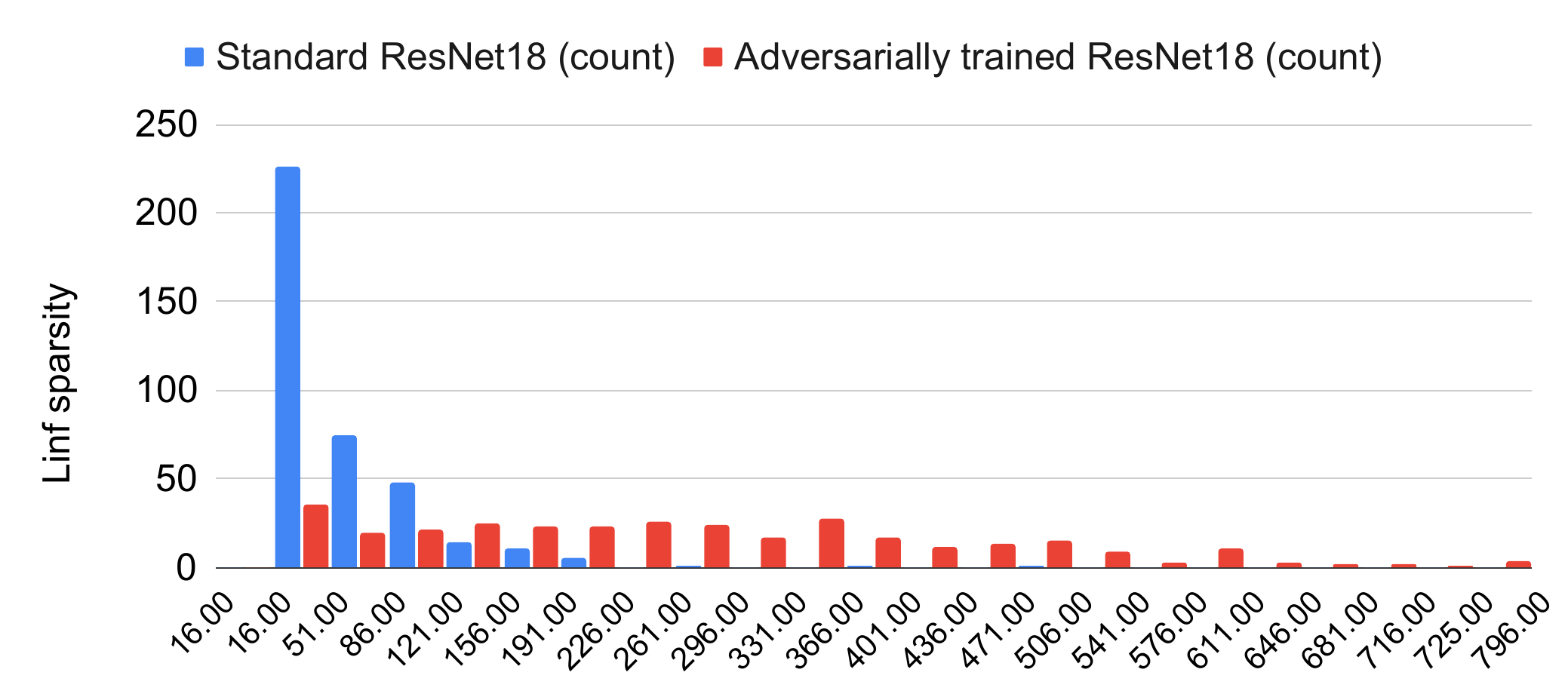}
         \caption{$L_\infty$ sparsity  \label{fig:histlinf}}
     \end{subfigure}
     \caption{Histogram of sparsity values for standard and adversarially trained ResNet18 models. It illustrates that sparsity variance with respect to input points is very large, especially for adversarially trained models. \label{fig:histogram}}
\end{figure*}

Interestingly, there seems to be a correlation between sparsity on the vanilla model, and accurate prediction on a robust (adversarially trained) model. Using a threshold-based classifier on the vanilla model sparsity to predict whether the robust classifier predicts them correctly, we can reach a precision of 67\% at Equal Error Rate. This demonstrates an additional property of adversarial sparsity: to discriminate points that are ``easy'' to learn with a robust decision boundary (the points with a sparse adversarial set) from harder ones.

\begin{figure}
     \centering
     \begin{subfigure}{0.49\textwidth}
         \centering
         \includegraphics[width=\textwidth]{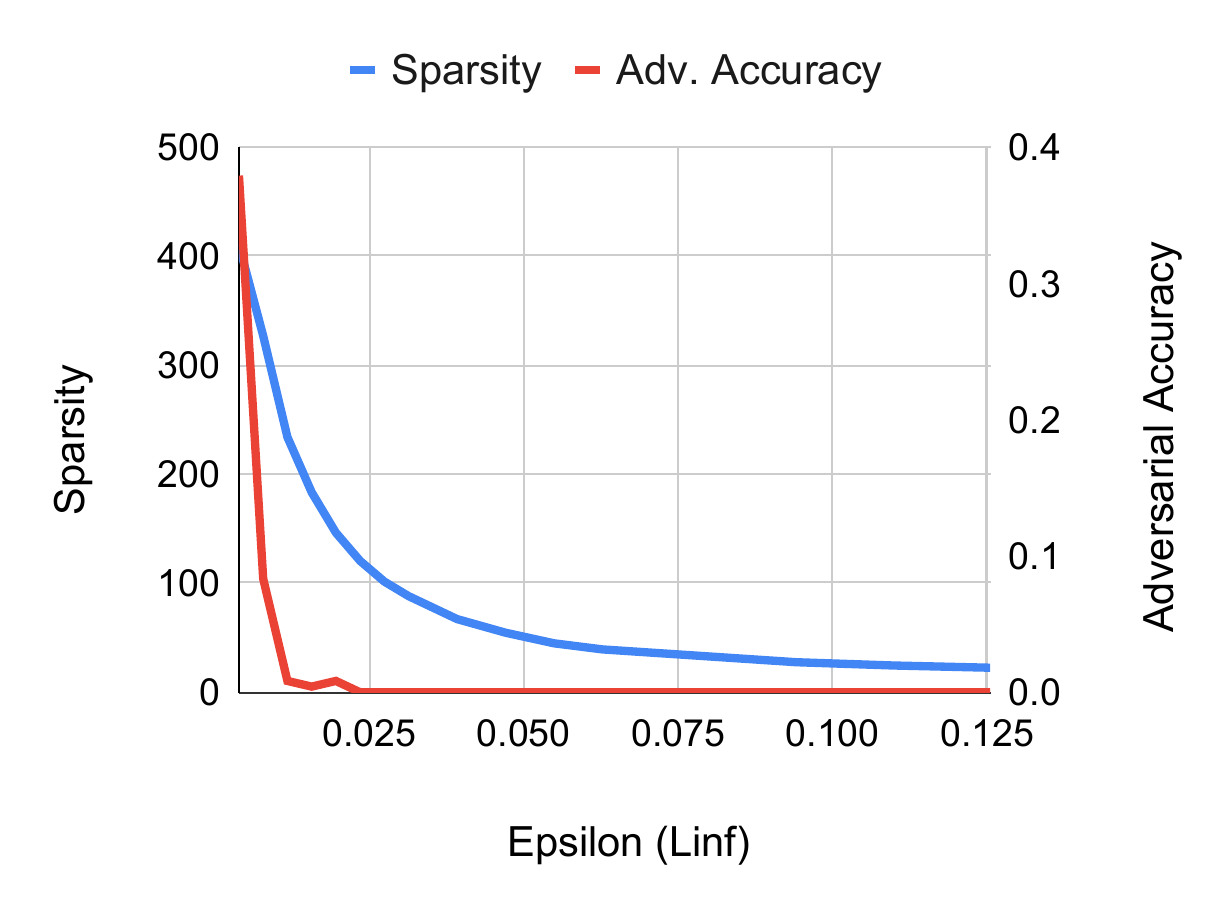}
         \caption{Standard ResNet18 model \label{fig:epsvarylinf}}
     \end{subfigure}
     \hfill
     \begin{subfigure}{0.49\textwidth}
         \centering
         \includegraphics[width=\textwidth]{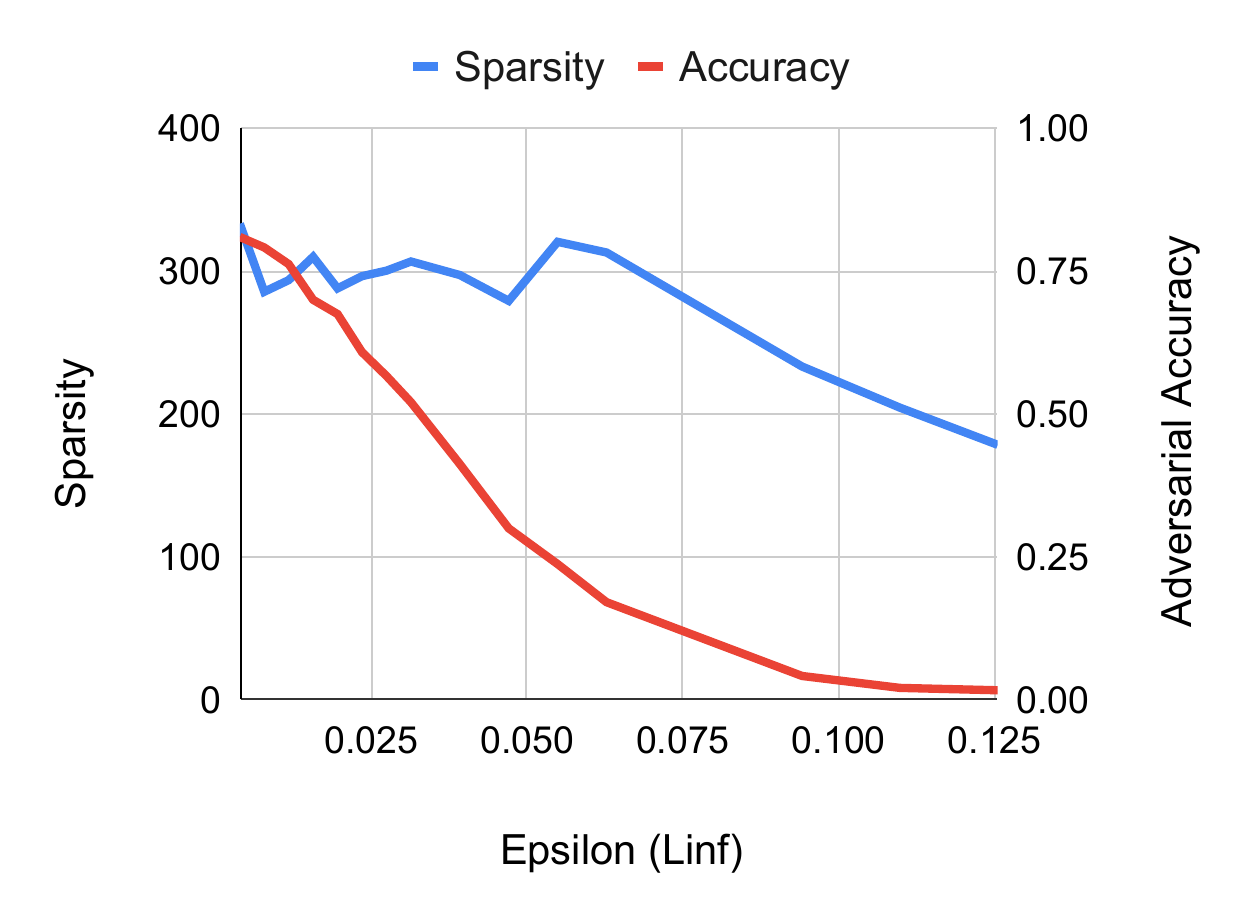}
         \caption{Adversarially trained model (training radius $.0314$) \label{fig:epsvarylinfadv}}
     \end{subfigure}
     \caption{Evolution of $L_\infty$ sparsity and adversarial accuracy of ResNet18 models as a function of the attack radius $\epsilon$ }
\end{figure}

\section{Influence of attack radius on $L_\infty$ sparsity}\label{apx:linfradius}
Figures \ref{fig:epsvarylinf} and \ref{fig:epsvarylinfadv} illustrate the evolution of $L_\infty$ sparsity as a function of the attack radius, for a standard and adversarially trained model respectively.

\section{A Theoretical Setting for sparsity}\label{apx:math}
We now formalize and prove the results mentioned in section \ref{sec:sparsadvsize} linking adversarial sparsity to the number of perturbations. $f$,$x$, $\epsilon$ and $n\geq3$ are fixed. 

\begin{proposition}
Let $\mathbb{\mu}$ be the probability measure associated with the uniform distribution over admissible set $\Delta$. Let $\mathcal{D}$ a distribution of sequences of subsets $\Delta^m \subset \Delta$, indexed on $M$. Assume $\mathcal{D}$ is such that the volume of $\Delta^m$ only depends on $m$: $$\frac{\mu(\Delta^m)}{\Delta}=g(m)$$ Assume $\text{Adv}(f,x,\epsilon) = \{\delta_j,1\leq j\leq k\}$ with $(\delta_j)$ uniformly sampled iid. over $\Delta$ Then we have: 
$$\mathbb{E}_{(\delta_j)\sim\mathcal{U}(\Delta)}[\text{AS}(f,\epsilon,x)] = \int_M(1-g(m))^kdm$$
\end{proposition}

\begin{proof}
Let us note $X_j = \inf\{m,  \delta_j\in\Delta^{m} \}$, such that 
$$\text{AS}(f,x,\epsilon, (\Delta^m)) = \inf_{1\leq j\leq k}X_j$$

Recall that:

\begin{equation}
\begin{aligned}
&\mathbb{E}_{(\delta_j)\sim\mathcal{U}(\Delta)}[\text{AS}(f,\epsilon,x)]\\
&= \mathbb{E}_{(\delta_j)\sim\mathcal{U}(\Delta)}[\mathbb{E}_{(\Delta^m)\sim\mathcal{D})}[\text{AS}(f,\epsilon,x,(\Delta^m))]] \\
&= \mathbb{E}_{(\Delta^m)\sim\mathcal{D})}[\mathbb{E}_{(\delta_j)\sim\mathcal{U}(\Delta)}[\text{AS}(f,\epsilon,x,(\Delta^m))]] \\
&= \mathbb{E}_{(\Delta^m)\sim\mathcal{D})}[\int_M \mathbb{P}_{(\delta_j)\sim\mathcal{U}(\Delta)}[\text{AS}(f,\epsilon,x,(\Delta^m))> m]dm] \\
\end{aligned}
\end{equation}

Note that 
\begin{equation}
\begin{aligned}
\mathbb{P}_{(\delta_j)\sim\mathcal{U}(\Delta)}[X_j> m] &= \mathbb{P}_{(\delta_j)\sim\mathcal{U}(\Delta)}[\delta_j\notin\Delta^{m}] \\
&= \frac{\mu(\bar{\Delta^m)}}{\mu(\Delta)}=1-g(m)
\end{aligned}
\end{equation}
Moreover, since $X_1,...,X_k$ are independent,
\begin{equation}
\begin{aligned}
&\mathbb{P}_{(\delta_j)\sim\mathcal{U}(\Delta)}[(\inf_{1\leq j\leq k}X_j)> m] \\
&= \mathbb{P}_{(\delta_j)\sim\mathcal{U}(\Delta)}[X_1>m \land ... \land X_k>m] \\
&= \Pi_{1\leq j\leq k}\mathbb{P}_{(\delta_j)\sim\mathcal{U}(\Delta)}[\delta_j\notin\Delta^{m}] \\
&=(1-g(m))^k
\end{aligned}
\end{equation}

It follows

\begin{equation}
\begin{aligned}
&\mathbb{E}_{(\delta_j)\sim\mathcal{U}(\Delta)}[\text{AS}(f,\epsilon,x)]\\
&= \mathbb{E}_{(\Delta^m)\sim\mathcal{D})}[\int_M (1-g(m))^kdm] \\
&= \int_M (1-g(m))^kdm \\
\end{aligned}
\end{equation}
\end{proof}

In the $L_2$ case $M = [0,\pi]$, $\Delta = S^n_2$  and $\Delta^\alpha$ is a spherical cap of angle $\alpha$. When $\alpha\leq \frac{\pi}{2}$ formula expressing the area of a spherical cap is derived in \cite{Li2011ConciseFF} and equal to:
\begin{equation}\label{eqn:capsurface}
    A(\alpha) = \frac{\pi^{\frac{n}{2}}}{\Gamma(\frac{n}{2})}I_{\sin{\alpha}^2}(\frac{n-1}{2},\frac{1}{2})
\end{equation}
where $\Gamma$ is the Gamma function and $I$ is the regularized incomplete beta function. Given that $I_1(a,b)=1$ it follows that: 
\begin{equation}\label{eqn:areacap2}
g(m) = g(\alpha) = t_\alpha:= I_{\sin{\alpha}^2}(\frac{n-1}{2},\frac{1}{2})
\end{equation}.

When $\frac{\pi}{2}<\alpha<\pi$ the cap of angle $\alpha$ is merely the complementary set on the sphere of the cap of angle $\pi-\alpha$. Therefore $g(\alpha)=1-g(\pi-\alpha)$ and:

\begin{equation}\label{eqn:sparsityproof2}
\begin{split}
\mathbb{E}[\text{AS}(f,x,\epsilon)]&= \int_0^\pi ((1-g(\alpha))^kd\alpha\\
&=\int_0^\frac{\pi}{2} ((1-t_\alpha)^k+t_\alpha^k)d\alpha
\end{split}
\end{equation}

In the $L_\infty$ case, $M=\{0,...,n\}$, $\Delta=\{\pm1\}^n$, $\Delta^m=\{\delta\in\Delta \mid \forall q>m\;\; \delta_{\sigma(q)}=u_{\sigma(q)}\}$. Therefore $g(m)=2^{m-n}$ and 
$$\mathbb{E}[\text{AS}(f,\epsilon,x)] = \sum_{m=0}^n(1-2^{m-n})^k = \sum_{m=0}^n(1-2^{-m})^k$$

We now show that $\mathbb{E}[\text{AS}(f,\epsilon,x)] \in \Theta(n-\log_2k)$. On the one hand, $$ln((1-2^{-m})^k=k.ln(1-2^{-m})\leq -k.2^{-m}$$
from which it follows:

\begin{equation}
\begin{aligned}
\mathbb{E}[\text{AS}(f,\epsilon,x)] &\leq \sum_{m=0}^{\lfloor\log_2k\rfloor}e^{-2^{\log_2k-m}}+\sum_{m=\lfloor\log_2k\rfloor+1}^{n}(1-2^{-m})^k\\
&\leq e^{-1}+...+e^{-m} + (n-\lfloor\log_2k\rfloor).1\\
&\leq n-\log_2k+1 + \frac{1}{e-1}\\
&= n-\log_2k + \frac{e}{e-1}\\
\end{aligned}
\end{equation}

On the other hand, $$k.ln(1-\frac{1}{k}) \geq k.(1-\frac{1}{1-\frac{1}{k}})= \frac{-k}{k-1}$$
Thus for $k\geq2$

\begin{equation}
\begin{aligned}
\mathbb{E}[\text{AS}(f,\epsilon,x)] &\geq \sum_{m=\lfloor\log_2k\rfloor+1}^{n}(1-2^{-m})^k \\
&\geq (n-\lfloor\log_2k\rfloor).(1-2^{-\lfloor\log_2k\rfloor-1})^k \\
&\geq (n-\log_2k).(1-\frac{1}{k})^k \\
&\geq (n-\log_2k).e^{-\frac{k}{k-1}} \\
&\geq (n-\log_2k).e^{-2} \\
\end{aligned}
\end{equation}

Since $(1-\frac{1}{2})^2=\frac{1}{4}$,$(1-\frac{1}{3})^3=\frac{8}{27}$ and for $k\geq4$ $e^{=\frac{k}{k-1}}\geq e^{\frac{4}{3}}>\frac{1}{4}$ we can conclude:

$$\mathbb{E}[\text{AS}(f,\epsilon,x)] \geq \frac{(n-\log_2k)}{4}$$
This also applies when $k=1$, since $\mathbb{E}[\text{AS}(f,\epsilon,x)]=n-\sum_{m\leq n}2^{-m}\geq n-2 \geq \frac{n}{4}$


\end{document}